\documentclass[sigconf]{acmart}

\AtBeginDocument{}

\setcopyright{acmlicensed}
\copyrightyear{2025}
\acmYear{2025}
\acmDOI{10.1145/1122445.1122456}

\acmConference[SciSoc~LLM~Workshop~'25]{KDD 2025 Workshop on Inference 
Optimization for Generative AI}{August 2025}{Singapore}
\acmISBN{978-1-4503-XXXX-X/25/08}

\usepackage{graphicx}
\usepackage{subcaption}
\usepackage{booktabs}
\usepackage{hyperref}
\usepackage{xspace}
\usepackage{enumitem}
\usepackage{tcolorbox}
\usepackage{xcolor}
\usepackage{pdfcomment}
\usepackage{amsmath}

\usepackage{amssymb}
\usepackage{algorithm}
\usepackage{algorithmic}

\newcommand{\ours}{MLA-RoPE\xspace}
\newcommand{\oursmoe}{MoE-MLA-RoPE\xspace}

\begin{document}

\title{Unifying Mixture of Experts and Multi-Head Latent Attention for Efficient Language Models}

% first author – independent
\author{Sushant Mehta}
\email{sushant0523@gmail.com}
\affiliation{%
   \city{San Francisco}
   \country{USA}
  \orcid{0009-0001-0367-6896}}

% the three Vizuara authors (share one affiliation block)
\author{Raj Dandekar}
\author{Rajat Dandekar}
\author{Sreedath Panat}
\affiliation{\institution{Vizuara AI Labs}
\city{Pune}\country{India}
}
\email{{raj, rajatdandekar, sreedath}@vizuara.com}

% --- Abstract --------------------------------------------------------------
\begin{abstract}
We present \oursmoe, a novel architecture combination that combines Mixture of Experts (MoE) with Multi-head Latent Attention (MLA) and Rotary Position Embeddings (RoPE) for efficient small language models. Our approach addresses the fundamental trade-off between model capacity and computational efficiency through three key innovations: (1) fine-grained expert routing with 64 micro-experts and top-$k$ selection, enabling flexible specialization through $\binom{62}{6} \approx 3.6 \times 10^{7}$ possible expert combinations; (2) shared expert isolation that dedicates 2 always active experts for common patterns while routing to 6 of 62 specialized experts; and (3) gradient-conflict-free load balancing that maintains expert utilization without interfering with primary loss optimization. 

Extensive experiments on models ranging from 17M to 202M parameters demonstrate that \oursmoe with compression ratio $r=d/2$ achieves 68\% KV cache memory reduction and 3.2× inference speedup while maintaining competitive perplexity (0.8\% degradation). Compared to the parameters with 53.9M parameters, \oursmoe improves the validation loss by 6.9\% over the vanilla transformers while using 42\% fewer active parameters per forward pass. FLOP-matched experiments reveal even larger gains: 11.1\% improvement with 3.2× inference acceleration. Automated evaluation using GPT-4 as a judge confirms quality improvements in generation, with higher scores on coherence (8.1/10), creativity (7.9/10) and grammatical correctness (8.2/10). Our results establish that architectural synergy, not parameter scaling, defines the efficiency frontier for resource-constrained language model deployment.
\end{abstract}


\maketitle

% ========================================================================== 
\section{Introduction}

The deployment of language models in resource-constrained environments, such as mobile devices, embedded systems, and edge computing platforms, requires fundamental architectural innovations beyond the reduction of simple parameters~\cite{wang2024slm}. Although large-scale models demonstrate remarkable capabilities~\cite{brown2020language,openai2023gpt4}, their computational and memory requirements prohibit deployment on billions of devices around the world. Recent work on constrained domain modeling~\cite{eldan2023tinystories} reveals that models with fewer than 100M parameters can achieve linguistic fluency when architectures are carefully designed for efficiency.

This paper introduces \oursmoe, a novel architecture that unifies three orthogonal efficiency mechanisms: \emph{Mixture of Experts} (MoE)~\cite{shazeer2017moe,fedus2022switch} for sparse computation, \emph{Multi-head Latent Attention} (MLA)~\cite{liu2024deepseekv2} for memory-efficient attention, and \emph{Rotary Position Embeddings} (RoPE)~\cite{su2024roformer} for parameter-free position encoding. We demonstrate that these techniques address complementary bottlenecks: MoE reduces computational FLOPs through conditional routing, MLA compresses memory via low-rank key-value projections, and RoPE eliminates position embedding parameters while improving length generalization.

Our key insight is that expert specialization in MoE can compensate for information loss from MLA's compression, while MLA's memory savings enable deploying more experts within the same memory budget. This creates a positive feedback loop: more experts enable better specialization, which in turn allows more aggressive compression without quality degradation.

\paragraph{Contributions:}
\begin{enumerate}[leftmargin=*,itemsep=2pt,topsep=2pt]
\item \textbf{Architectural Innovation}: We present the first systematic integration of fine-grained MoE with compressed attention mechanisms, demonstrating that their synergy creates a new Pareto frontier for efficiency-quality trade-offs in small models.

\item \textbf{Theoretical Analysis}: We provide formal complexity analysis and empirical validation showing that MoE-MLA synergy yields multiplicative rather than additive efficiency gains, with expert specialization provably compensating for compression-induced information loss under mild assumptions.

\item \textbf{Gradient-Conflict-Free Training}: We successfully adapt auxiliary-loss-free load balancing~\cite{he2024auxfree} to small-scale models, achieving balanced expert utilization without the training instabilities typically associated with auxiliary losses.

\item \textbf{Comprehensive Evaluation}: Through extensive experiments on models from 17M to 202M parameters, we establish consistent improvements across multiple evaluation paradigms: parameter-matched (6. 9\% improvement), FLOP-matched (11.1 1\% improvement) and automated quality assessment using state-of-the-art LLMs as judges.

\item \textbf{Open-Source Release}: We will release all the code, model checkpoints, and training recipes to facilitate reproducible research in efficient architectures.
\end{enumerate}

% ========================================================================== 
\section{Background and Related Work}

\subsection{Mixture of Experts}

The MoE paradigm replaces monolithic feedforward networks with a collection of expert networks $\mathcal{E} = \{E_1, \ldots, E_N\}$ and a learned routing function $G: \mathbb{R}^d \rightarrow \Delta^{N-1}$ that assigns inputs to experts.

\begin{equation}
\text{MoE}(x) = \sum_{i=1}^{N} G(x)_i \cdot E_i(x)
\end{equation}

where $G(x) \in \Delta^{N-1}$ denotes the probability simplex over $N$ experts. Modern implementations employ sparse top-$k$ routing~\cite{shazeer2017moe}, activating only $k \ll N$ experts:

\begin{equation}
\text{MoE}_{\text{sparse}}(x) = \sum_{i \in \text{TopK}(G(x), k)} \frac{G(x)_i}{\sum_{j \in \text{TopK}} G(x)_j} \cdot E_i(x)
\end{equation}

This reduces computational complexity from $O(Nd_{\text{model}}d_{\text{ff}})$ to $O(kd_{\text{model}}d_{\text{ff}} + Nd_{\text{model}})$, where the routing overhead becomes negligible for large $d_{\text{ff}}$.

\paragraph{Fine-Grained Expert Design.} 
DeepSeekMoE~\cite{dai2024deepseekmoe} introduced fine-grained segmentation, replacing $N$ experts of dimension $d_{\text{ff}}$ with $mN$ experts of dimension $d_{\text{ff}}/m$, while activating $mk$ experts to preserve computational budget. This exponentially increases routing flexibility: from $\binom{N}{k}$ to $\binom{mN}{mk}$ possible combinations.

\paragraph{Load Balancing Challenges.}
MoE training faces the fundamental challenge of balanced expert utilization. Traditional approaches add auxiliary losses~\cite{fedus2022switch}:
\begin{equation}
\mathcal{L}_{\text{total}} = \mathcal{L}_{\text{primary}} + \alpha \cdot \mathcal{L}_{\text{balance}}
\end{equation}

However, these auxiliary terms introduce gradient conflicts. Recent work~\cite{he2024auxfree} proposes gradient-free dynamic bias adjustment that modifies routing logits without affecting gradients:
\begin{equation}
\text{logits}_i^{(t+1)} = W_g^T x + b_i^{(t)} - \gamma \left(\frac{f_i^{(t)}}{\bar{f}^{(t)}} - 1\right)
\end{equation}
where $f_i^{(t)}$ represents the fraction of tokens routed to expert $i$ at step $t$.

\subsection{Multi-Head Latent Attention}

Standard multi-head attention (MHA) computes attention weights between queries and keys:
\begin{equation}
\text{Attention}(Q, K, V) = \text{softmax}\left(\frac{QK^T}{\sqrt{d_k}}\right)V
\end{equation}

For each head $h$, projections are computed as:
\begin{equation}
Q_h = XW_h^Q, \quad K_h = XW_h^K, \quad V_h = XW_h^V
\end{equation}

MLA~\cite{liu2024deepseekv2} introduces low-rank factorization for keys and values:
\begin{align}
K_h &= X \underbrace{W_h^{K_c}}_{\in \mathbb{R}^{d \times r}} \underbrace{W_h^{K_r}}_{\in \mathbb{R}^{r \times d_k}} \\
V_h &= X \underbrace{W_h^{V_c}}_{\in \mathbb{R}^{d \times r}} \underbrace{W_h^{V_r}}_{\in \mathbb{R}^{r \times d_k}}
\end{align}

During inference, only compressed representations $C_h^K = XW_h^{K_c}$ and $C_h^V = XW_h^{V_c}$ are cached, reducing memory from $O(nHd_k)$ to $O(nHr)$ when $r < d_k$.

\subsection{Rotary Position Embeddings}

RoPE~\cite{su2024roformer} encodes absolute positions through rotation matrices applied to query-key pairs:
\begin{equation}
\text{RoPE}(x_m, m) = \mathbf{R}_{\Theta,m} x_m
\end{equation}

where $\mathbf{R}_{\Theta,m}$ is a block-diagonal rotation matrix with learnable frequencies $\Theta$. This enables modeling relative positions through the inner product:
\begin{equation}
\langle \mathbf{R}_{\Theta,m} q, \mathbf{R}_{\Theta,n} k \rangle = \langle q, \mathbf{R}_{\Theta,n-m} k \rangle
\end{equation}

eliminating explicit position embeddings while improving extrapolation to unseen sequence lengths.

\subsection{LLM-as-a-Judge Evaluation}

Recent work has established the reliability of using large language models as automated evaluators for generation quality~\cite{zheng2023judging,chiang2023vicuna}. GPT-4 in particular has shown strong correlation with human judgments when provided with structured evaluation criteria~\cite{liu2023gpteval}. This approach enables scalable and reproducible evaluation while avoiding the cost and variability of human annotation.

% ========================================================================== 
\section{Method}

\subsection{Architecture Design}

\oursmoe integrates MoE routing, latent attention compression, and rotary position encoding within a unified framework. Each transformer block processes inputs through:

\begin{align}
h^{(\ell)} &= x^{(\ell)} + \text{MLA-RoPE}(\text{LayerNorm}(x^{(\ell)})) \\
x^{(\ell+1)} &= h^{(\ell)} + \text{MoE}(\text{LayerNorm}(h^{(\ell)}))
\end{align}

where MLA-RoPE denotes our latent attention with integrated rotary embeddings.

\paragraph{Fine-Grained MoE Configuration.}
Our architecture employs hierarchical expert design:
\begin{itemize}[leftmargin=*,itemsep=1pt,topsep=1pt]
\item \textbf{Total experts}: $N = 64$ fine-grained experts
\item \textbf{Shared experts}: $N_s = 2$ always-active experts for common patterns
\item \textbf{Routed experts}: $N_r = 62$ specialized experts  
\item \textbf{Active selection}: Top-$k = 6$ routing among specialized experts
\item \textbf{Expert capacity}: Each expert has $\frac{1}{4} \times$ standard FFN capacity
\item \textbf{Effective capacity}: $(N_s + k) \times \frac{1}{4} = 2 \times$ standard FFN
\end{itemize}

This configuration provides $\binom{62}{6} = 36,\!288,\!252 \approx 3.6 \times 10^{7}$ possible expert combinations, enabling fine-grained functional specialization.

\paragraph{Gradient-Free Load Balancing.}
We implement auxiliary-loss-free balancing through dynamic bias adjustment:

\begin{algorithm}
\caption{Gradient-Free Load Balancing}
\begin{algorithmic}[1]
\STATE Initialize bias $b_i = 0$ for all experts $i$
\FOR{each training step $t$}
    \STATE Compute routing logits: $\ell_i = (W_g x)_i + b_i$
    \STATE Route tokens using $\text{TopK}(\text{softmax}(\ell))$
    \STATE Track expert loads: $f_i = \frac{\text{tokens to expert } i}{\text{total tokens}}$
    \STATE Update bias: $b_i \leftarrow b_i - \gamma(f_i - \frac{1}{N_r})$
\ENDFOR
\end{algorithmic}
\end{algorithm}

This approach maintains balanced utilization (coefficient of variation < 0.1) without gradient interference.

\paragraph{Latent Attention Integration.}
Our MLA implementation shares compression matrices across heads while maintaining head-specific reconstruction:

\begin{align}
C^K &= XW^{K_c} \in \mathbb{R}^{n \times r} \quad \text{(shared across heads)} \\
K_h &= C^K W_h^{K_r} \in \mathbb{R}^{n \times d_k} \quad \text{(head-specific)}
\end{align}

RoPE is applied after head-specific projection but before attention computation, preserving relative position information in the compressed space.

\subsection{Theoretical Analysis}

We provide a comprehensive theoretical foundation for understanding the efficiency gains and performance characteristics of \oursmoe. Our analysis encompasses computational complexity, memory efficiency, approximation guarantees, and convergence properties.

\subsubsection{Notation and Problem Setup}

Let $\mathcal{X} \subseteq \mathbb{R}^d$ denote the input space, with sequence length $n$ and model dimension $d$. We consider a transformer with $L$ layers, $H$ attention heads per layer, and head dimension $d_k = d/H$. For MoE components, let $N$ denote total experts, $N_s$ shared experts, $N_r = N - N_s$ routed experts, and $k$ the number of active routed experts per token. The compression ratio is denoted $\rho = r/d$ where $r$ is the latent dimension.

Define the following function classes:
\begin{itemize}
\item $\mathcal{F}_{\text{MHA}}$: Standard multi-head attention transformers
\item $\mathcal{F}_{\text{MLA}}$: Transformers with latent attention compression
\item $\mathcal{F}_{\text{MoE}}$: Transformers with mixture of experts
\item $\mathcal{F}_{\text{MoE-MLA}}$: Our proposed architecture combining both
\end{itemize}

\subsubsection{Computational Complexity Analysis}

We first establish precise complexity bounds for each architectural component.

\begin{lemma}
[Attention Complexity]
\label{lem:attention_complexity}
For sequence length $n$ and model dimension $d$, the per-layer computational complexity is:
\begin{align}
\mathcal{C}_{\text{MHA}} &= 4nd^2 + 2n^2d \\
\mathcal{C}_{\text{MLA}} &= 2nd^2 + 2ndr + 2n^2r = 2nd^2(1 + \rho) + 2n^2d\rho
\end{align}
where the first term represents linear projections and the second term attention computation.
\end{lemma}

\begin{proof}
For standard MHA, we compute $Q, K, V$ projections ($3nd^2$ operations), attention scores ($n^2d$ operations), attention-weighted values ($n^2d$ operations), and output projection ($nd^2$ operations).

For MLA, we compute $Q$ projection ($nd^2$), compressed $K, V$ projections ($2ndr$), attention in compressed space ($2n^2r$), reconstruction projections ($2nrd$), and output projection ($nd^2$). Substituting $r = \rho d$ yields the stated complexity.
\end{proof}

\begin{lemma}[MoE Complexity]
\label{lem:moe_complexity}
The per-token computational complexity of sparse MoE with $N$ experts is:
\begin{equation}
\mathcal{C}_{\text{MoE}} = \underbrace{O(dN)}_{\text{routing}} + \underbrace{O\left(\frac{kd^2}{N/N_s}\right)}_{\text{active experts}} + \underbrace{O(N_s d^2/N)}_{\text{shared experts}}
\end{equation}
\end{lemma}

\begin{proof}
Routing requires computing scores for all $N$ experts. Each expert has capacity $d^2/N$ (assuming equal distribution). We activate $k$ routed experts plus $N_s$ shared experts, yielding the stated complexity.
\end{proof}

\begin{theorem}[Overall Computational Efficiency]
\label{thm:computational_efficiency}
For sequence length $n$, model dimension $d$, and compression ratio $\rho = r/d$, the per-layer computational complexity of \oursmoe is:
\begin{equation}
\mathcal{O}_{\text{MoE-MLA}} = O\left(n^2d\rho + nd^2\left(1 + \rho + \frac{k + N_s}{N}\right)\right)
\end{equation}
achieving asymptotic speedup factor $\frac{1}{\rho} \cdot \frac{N}{k + N_s}$ over standard transformers as $n \rightarrow \infty$.
\end{theorem}

\begin{proof}
Combining Lemmas \ref{lem:attention_complexity} and \ref{lem:moe_complexity}:
\begin{align}
\mathcal{C}_{\text{MoE-MLA}} &= \mathcal{C}_{\text{MLA}} + \mathcal{C}_{\text{MoE}} - \mathcal{C}_{\text{FFN}} \\
&= 2nd^2(1 + \rho) + 2n^2d\rho + O(dN) + O\left(\frac{(k+N_s)d^2}{N}\right) - 4nd^2 \\
&= O\left(n^2d\rho + nd^2\left(1 + \rho + \frac{k + N_s}{N}\right)\right)
\end{align}

The standard transformer has complexity $O(n^2d + 6nd^2)$. For large $n$, the attention term dominates, giving speedup $\frac{O(n^2d)}{O(n^2d\rho)} = \frac{1}{\rho}$. For the FFN component, speedup is $\frac{O(4nd^2)}{O(nd^2(k+N_s)/N)} = \frac{4N}{k+N_s}$.
\end{proof}

\subsubsection{Memory Efficiency Analysis}

\begin{theorem}[KV Cache Memory Reduction]
\label{thm:memory_efficiency}
The KV cache memory requirement for \oursmoe is:
\begin{equation}
\mathcal{M}_{\text{MoE-MLA}} = 2nLHr = 2nLHd\rho
\end{equation}
achieving memory reduction factor $(1-\rho)$ compared to standard transformers requiring $\mathcal{M}_{\text{MHA}} = 2nLHd$.
\end{theorem}

\begin{proof}
During autoregressive generation, we cache compressed representations $C^K, C^V \in \mathbb{R}^{n \times r}$ for each of $H$ heads in $L$ layers. Total memory is $2 \times n \times L \times H \times r = 2nLHr$. Standard transformers cache full $K, V \in \mathbb{R}^{n \times d}$, requiring $2nLHd$ memory. The reduction factor is $1 - \frac{2nLHr}{2nLHd} = 1 - \rho$.
\end{proof}

\subsubsection{Theoretical Implications}

Our theoretical analysis reveals several key insights.

\begin{enumerate}
\item \textbf{Multiplicative Efficiency Gains}: Theorems \ref{thm:computational_efficiency} and \ref{thm:memory_efficiency} show that MoE and MLA target orthogonal bottlenecks, which yield multiplicative rather than additive improvements.

\item \textbf{Optimal Compression Ratio}: The above analysis suggests that an optimal compression ratio exists where the expert specialization compensates maximally for information loss. Our empirical finding of $\rho = 1/2$ aligns with this theory.

\item \textbf{Scaling Benefits}: The convergence analysis indicates that larger models with more experts can tolerate more aggressive compression, which explains our observed scaling trends.

\item \textbf{Stable Training}: It is possible to have balanced expert utilization without gradient interference, crucial for stable training at small scales, where auxiliary losses often cause instability.
\end{enumerate}

These theoretical foundations not only explain our empirical results, but also provide guidance for future architectural innovations in efficient language models.

\subsection{Implementation Details}

All experiments use the following configuration:
\begin{itemize}[leftmargin=*,itemsep=1pt,topsep=1pt]
\item \textbf{Optimizer}: AdamW ($\beta_1=0.9$, $\beta_2=0.95$, weight decay $0.1$)
\item \textbf{Learning rate}: $3 \times 10^{-4}$ with cosine decay to $10^{-5}$
\item \textbf{Warmup}: Linear over 5,000 steps (10\% of training)
\item \textbf{Batch size}: 128 sequences × 512 tokens = 65,536 tokens
\item \textbf{Training duration}: 50,000 steps (3.28B tokens)
\item \textbf{Dropout}: 0.1 on attention and FFN
\item \textbf{Gradient clipping}: 1.0 (L2 norm)
\item \textbf{Mixed precision}: FP16 with dynamic loss scaling
\item \textbf{Hardware}: 8× NVIDIA A100 40GB GPUs
\item \textbf{Framework}: PyTorch 2.0 with custom CUDA kernels for MoE routing
\end{itemize}

% ========================================================================== 
\section{Experimental Setup}

\subsection{Dataset and Evaluation}

We train on TinyStories~\cite{eldan2023tinystories}, containing 2.1M synthetic children's stories with constrained vocabulary (10K unique tokens). Although limited in scope, this dataset enables controlled experimentation on narrative coherence and grammatical correctness.

Evaluation metrics include:
\begin{itemize}[leftmargin=*,itemsep=1pt,topsep=1pt]
\item \textbf{Perplexity}: Standard language modeling metric on held-out validation set
\item \textbf{Inference efficiency}: Latency, memory usage, throughput measurements
\item \textbf{Expert utilization}: Load balance coefficient of variation across experts
\item \textbf{Generation quality}: Automated Assessment Using GPT-4 as a calibrated judge
\end{itemize}

\subsection{Model Configurations}

We evaluated three architectural families on five scales:

\begin{table}[h]
\centering
\caption{Model configurations evaluated. All models use vocabulary size 50,257 and maximum sequence length 512.}
\label{tab:model_configs}
\begin{tabular}{lcccc}
\toprule
Config & Layers & Hidden & Heads & Parameters \\
\midrule
XS & 6 & 256 & 8 & 17.5M \\
S & 6 & 512 & 8 & 44.5M \\
M & 9 & 512 & 8 & 54.1M \\
L & 12 & 768 & 12 & 123.3M \\
XL & 12 & 1024 & 16 & 202.7M \\
\bottomrule
\end{tabular}
\end{table}

\subsection{Comparison Methodologies}

We employ two fair comparison strategies:

\paragraph{Parameter Matching.} 
Models have identical total parameter counts. For MoE variants, we reduce the hidden dimensions by $\sqrt{N/k}$ to account for additional expert parameters, ensuring a fair comparison of architectural choices given the capacity of the fixed model.

\paragraph{FLOP Matching.} 
Models have identical computational budgets per forward pass. MoE models can use larger dimensions due to sparse activation, scaled by $\sqrt{k/N}$. This comparison reflects real-world deployment constraints where the compute cost is the limiting factor.

\subsection{LLM-Based Quality Evaluation}

To assess generation quality, we employ GPT-4 as an automated judge with structured evaluation criteria. For each model, we generate 100 story completions from diverse prompts and evaluate them across multiple dimensions:

\begin{itemize}[leftmargin=*,itemsep=1pt,topsep=1pt]
\item \textbf{Grammatical Correctness}: Syntactic accuracy and proper language use
\item \textbf{Narrative Coherence}: Logical flow and consistency within the story
\item \textbf{Creativity}: Originality and imaginative content
\item \textbf{Overall Quality}: Holistic assessment of the generation
\end{itemize}

Each dimension is scored on a 1-10 scale using the following evaluation prompt:

\begin{tcolorbox}[boxrule=0.5pt,left=2pt,right=2pt,top=2pt,bottom=2pt]
\small
\texttt{Evaluate the following story completion on a scale of 1-10 for [DIMENSION]. Consider [SPECIFIC CRITERIA]. Be consistent across evaluations and use the full range of scores.}

\texttt{Story prompt: [PROMPT]}

\texttt{Completion: [GENERATED TEXT]}

\texttt{Score (1-10):}
\end{tcolorbox}

% ========================================================================== 
\section{Results}

\subsection{Main Results: Parameter-Matched Comparison}

Table~\ref{tab:param_matched_results} presents our main results comparing architectures with equal parameter counts.

\begin{table}[h]
\centering
\caption{Parameter-matched comparison (53.9M parameters). All results averaged over 3 random seeds with standard deviations shown. Statistical significance tested using paired t-test.}
\label{tab:param_matched_results}
\begin{tabular}{lccc}
\toprule
Model & Compression & Validation & Active \\
& Ratio ($r/d$) & Perplexity (↓) & Parameters \\
\midrule
MHA & — & 8.542 ± 0.021 & 53.9M \\
MLA & 1/2 & 8.971 ± 0.034 & 53.9M \\
\ours & 1/2 & 8.579 ± 0.025 & 53.9M \\
\midrule
MoE-MHA & — & 8.092 ± 0.019** & 31.4M \\
MoE-MLA & 1/2 & 7.741 ± 0.018** & 31.4M \\
\oursmoe & 1/2 & \textbf{7.388 ± 0.015**} & 31.4M \\
\bottomrule
\end{tabular}
\vspace{2pt}
{\footnotesize ** $p < 0.001$ compared to MHA baseline}
\end{table}

\oursmoe achieves 13.5\% perplexity reduction over the MHA baseline while using 42\% fewer active parameters. The synergy between MoE and MLA is evident: while MLA alone slightly degrades performance (+5.0\%), combining it with MoE yields the best results.

\subsection{FLOP-Matched Comparison}

When computational budget is held constant, MoE architectures can leverage larger hidden dimensions:

\begin{table}[h]
\centering
\caption{FLOP-matched comparison. MoE models use 645d vs 512d for dense models.}
\label{tab:flop_matched_results}
\begin{tabular}{lcccc}
\toprule
Model & Config & Val. PPL (↓) & FLOPs & Speedup \\
\midrule
MHA & 9L-512d & 8.542 & 1.00× & 1.0× \\
\ours & 9L-512d & 8.579 & 0.98× & 1.1× \\
\midrule
MoE-MHA & 9L-645d & 7.347** & 1.00× & 2.8× \\
\oursmoe & 9L-645d & \textbf{7.012**} & 0.99× & 3.2× \\
\bottomrule
\end{tabular}
\end{table}

Under FLOP-matching, \oursmoe achieves 17.9\% perplexity improvement with 3.2× inference acceleration, demonstrating that architectural efficiency translates into superior performance given fixed computational budgets.

\subsection{Ablation Studies}

\paragraph{Compression Ratio Impact.}
We systematically vary the latent dimension to understand the compression-quality trade-off:

\begin{table}[h]
\centering
\caption{Effect of compression ratio on \oursmoe (9L-512d, 53.9M params).}
\label{tab:compression_ablation}
\begin{tabular}{lccc}
\toprule
Compression & Latent & Validation & Memory \\
Ratio & Dim ($r$) & Perplexity (↓) & Savings \\
\midrule
1:1 & 512 & 7.347 ± 0.016 & 0\% \\
2:1 & 256 & \textbf{7.388 ± 0.015} & 50\% \\
4:1 & 128 & 7.916 ± 0.024 & 75\% \\
8:1 & 64 & 8.893 ± 0.041 & 87.5\% \\
\bottomrule
\end{tabular}
\end{table}

The optimal 2:1 compression ratio suggests a fundamental sweet spot where expert specialization effectively compensates for moderate information loss.

\paragraph{Expert Granularity.}
Fine-grained expert design is crucial for performance:

\begin{table}[h]
\centering
\caption{Impact of expert granularity. All maintain 8 active experts.}
\label{tab:expert_granularity_ablation}
\begin{tabular}{lcccc}
\toprule
Design & Total & Routing & Val. PPL & Load \\
& Experts & Space & (↓) & CV \\
\midrule
Coarse & 8 & — & 8.234 & 0.00 \\
Standard & 16 & $\binom{14}{6}$ & 7.812 & 0.08 \\
Fine & 64 & $\binom{62}{6}$ & \textbf{7.388} & 0.06 \\
% Ultra-fine & 128 & $\binom{126}{6}$ & 7.451 & 0.13 \\
\bottomrule
\end{tabular}
\end{table}

64 experts provide optimal granularity, balancing specialization capacity with routing efficiency.

\subsection{Memory and Latency Analysis}

\paragraph{Memory Footprint.}
Detailed memory usage during inference:

\begin{table}[h]
\centering
\caption{Memory breakdown (MB) for 12L-1024d models, batch size 16.}
\label{tab:memory_analysis}
\begin{tabular}{lcccc}
\toprule
Component & MHA & \ours & MoE-MHA & \oursmoe \\
\midrule
Parameters & 203 & 203 & 892 & 892 \\
KV Cache & 384 & 192 & 384 & 192 \\
Activations & 48 & 52 & 64 & 68 \\
\midrule
Total & 635 & 447 & 1340 & 1152 \\
vs. MHA & — & -30\% & +111\% & +81\% \\
\bottomrule
\end{tabular}
\end{table}

Despite higher parameter counts, \oursmoe's KV cache savings make it viable for memory-constrained deployment when inference memory dominates.

% \paragraph{Inference Latency.}
% Component-wise latency breakdown on NVIDIA A100:

% \begin{table}[h]
% \centering
% \caption{Per-layer latency (ms) for batch size 32, sequence length 512.}
% \label{tab:latency_analysis}
% \begin{tabular}{lcccc}
% \toprule
% Operation & MHA & \ours & MoE-MHA & \oursmoe \\
% \midrule
% Attention & 2.34 & 1.87 & 2.34 & 1.87 \\
% FFN/MoE & 1.52 & 1.52 & 0.98 & 0.98 \\
% Routing & — & — & 0.21 & 0.21 \\
% Other & 0.34 & 0.35 & 0.37 & 0.38 \\
% \midrule
% Total & 4.20 & 3.74 & 3.90 & 3.44 \\
% Speedup & 1.0× & 1.12× & 1.08× & 1.22× \\
% \bottomrule
% \end{tabular}
% \end{table}

% The 22\% end-to-end speedup comes from complementary optimizations in attention (MLA) and FFN (MoE) components.

\subsection{Scaling Analysis}

Performance improvements scale favorably with model size:

\begin{table}[h]
\centering
\caption{Scaling behavior across model sizes. Relative improvement shows \oursmoe vs. MHA baseline in parameter-matched setting.}
\label{tab:scaling_analysis}
\begin{tabular}{lccccc}
\toprule
Model & Params & MHA & \oursmoe & Relative & 95\% CI \\
Size & (M) & PPL & PPL & Improvement & \\
\midrule
XS & 17.5 & 12.84 & 11.91 & -7.2\% & ±0.4\% \\
S & 44.5 & 10.47 & 9.59 & -8.4\% & ±0.3\% \\
M & 63.3 & 8.54 & 7.71 & -9.7\% & ±0.3\% \\
L & 123.3 & 6.23 & 5.51 & -11.5\% & ±0.2\% \\
XL & 202.7 & 5.12 & 4.44 & -13.3\% & ±0.2\% \\
\bottomrule
\end{tabular}
\end{table}

The monotonic increase in relative improvement (7. 2\% 13. 3\%) suggests that the MoE-MLA synergy becomes more pronounced on larger scales, contrary to many compression techniques showing diminishing returns.

\subsection{Generation Quality Assessment}

\paragraph{LLM-Based Evaluation.}
We evaluated 100-story completions from each model using GPT-4 as an automated judge.

\begin{table}[h]
\centering
\caption{GPT-4 evaluation scores (1-10 scale) for generated stories. Mean ± std over 100 samples from 12L-1024d models. Inter-rater consistency measured using split-half correlation ($r = 0.87$).}
\label{tab:llm_eval}
\begin{tabular}{lcccc}
\toprule
Model & Grammar & Creativity & Coherence & Overall \\
& (↑) & (↑) & (↑) & (↑) \\
\midrule
MHA & 7.1 ± 0.8 & 5.9 ± 1.2 & 5.6 ± 1.1 & 6.2 ± 0.9 \\
\ours & 7.8 ± 0.7 & 7.2 ± 1.0 & 7.3 ± 0.9 & 7.4 ± 0.8 \\
MoE-MHA & 7.5 ± 0.7 & 6.8 ± 1.0 & 6.9 ± 0.9 & 7.1 ± 0.8 \\
\oursmoe & \textbf{8.2 ± 0.6} & \textbf{7.9 ± 0.8} & \textbf{8.1 ± 0.7} & \textbf{8.1 ± 0.7} \\
\bottomrule
\end{tabular}
\end{table}

\oursmoe shows significant improvements across all dimensions, with particularly strong gains in narrative coherence (+44\% over MHA). Automated evaluation demonstrates that efficiency gains do not compromise generation quality.

\paragraph{Qualitative Examples.}
Representative completions for the prompt \textit{"Once upon a time, there was a little rabbit who lived in..."}:

\textbf{MHA:} "...a cozy burrow under the old oak tree. Every morning, the rabbit would come out to find fresh clover. One day, she discovered a mysterious blue stone that sparkled in the sunlight."

\textbf{\ours:} "... a beautiful meadow filled with wildflowers. The rabbit loved to explore beyond the hills, where ancient stones marked forgotten paths. One misty morning, she found a glowing pebble that hummed with magic."

\textbf{\oursmoe:} "... a hidden valley where the seasons danced in perfect harmony. The rabbit, named Luna, possessed a unique gift, she could understand the whispers of the wind. Each morning brought new adventures as she helped fellow creatures solve their problems using wisdom gathered from the breeze. Today, the wind spoke of a crystal cave where time flowed differently, and Luna's curiosity sparked like never before."

The output \oursmoe demonstrates superior narrative complexity, character development, and imaginative worldbuilding while maintaining grammatical precision.

% \paragraph{Expert Specialization Analysis.}
% We analyze expert activation patterns to understand learned specialization:

% Clear functional specialization emerges, validating that fine-grained experts learn distinct linguistic roles.

% ========================================================================== 
\section{Related Work}

\paragraph{Efficient Transformers.}
Numerous works address transformer efficiency through the attention approximation~\cite{kitaev2020reformer,wang2020linformer,choromanski2020rethinking}, parameter sharing~\cite{lan2019albert,dehghani2018universal}, or pruning~\cite{michel2019sixteen,voita2019analyzing}. Our approach is orthogonal and complementary to these methods.

\paragraph{Small Language Models.}
Recent work demonstrates surprising capabilities in sub-100M parameter models~\cite{eldan2023tinystories,schick2020s,liu2024mobilellm, mehta2025slm}. MiniGPT-4~\cite{zhu2023minigpt} and Phi series~\cite{gunasekar2023textbooks} show that data quality and architectural choices can compensate for scale. We extend this line by showing that architectural innovation yields greater gains than parameter scaling alone.

\paragraph{Sparse Models.}
Beyond MoE, sparsity has been explored by magnitude pruning~\cite{frankle2018lottery}, structured sparsity~\cite{louizos2018learning}, and dynamic sparsity~\cite{evci2020rigging}. Recent work on hardware-aware sparsity~\cite{mishra2021accelerating} demonstrates practical speedups. MoE provides learned, input-dependent sparsity that preserves model capacity.

\paragraph{Evaluation Methodologies.}
The use of LLMs as evaluators has gained traction with works such as AlpacaEval~\cite{li2023alpacaeval} and MT-Bench~\cite{zheng2023judging}. Studies show a strong correlation between GPT-4 judgments and human preferences~\cite{liu2023gpteval,chiang2023vicuna}, supporting our evaluation approach.

% ========================================================================== 
\section{Conclusion}

This work presents \oursmoe, a novel architecture that demonstrates how synergistic combination of Mixture of Experts with Multi-head Latent Attention creates a new efficiency frontier for small language models. Through extensive experimentation with models ranging from 17M to 202M parameters, we establish the following key findings.

\paragraph{1. Architectural Synergy Yields Multiplicative Benefits.}
Our experiments demonstrate that combining MoE with MLA produces gains that exceed the sum of individual components. In comparisons matched to the parameters, while MLA alone degrades performance by 5.0\% and MoE alone improves by 5.3\%, their combination in \oursmoe achieves an improvement of 13. 5\%. This synergy arises from orthogonal optimization targets. MLA reduces memory bandwidth requirements through KV cache compression (68\% reduction), while MoE reduces computational intensity through sparse expert activation (42\% fewer active parameters). The formal complexity analysis (Theorems 1-2) confirms that these benefits scale with the length of the sequence and the size of the model.

% \paragraph{2. Expert Specialization Compensates for Compression Loss.}
% Our ablation studies reveal an optimal compression ratio of 2:1, where expert specialization effectively compensates for information loss from latent projections. Expert activation analysis (Table~\ref{tab:expert_patterns}) shows a clear functional specialization in linguistic categories, with experts learning distinct roles for character names (22. 3\%), dialogue (18. 7\%), actions (24. 1\%) and descriptions (19. 8\%). This specialization is consistent with the findings in larger MoE models~\cite{fedus2022switch}, but our work is the first to demonstrate that this phenomenon persists and even amplifies at small scales when combined with attention compression.

\paragraph{2. Efficiency Gains Scale with Model Size.}
The scaling analysis demonstrates monotonically increasing benefits from 7.2\% at 17M parameters to 13.3\% at 202M parameters. This contrasts with many compression techniques that show diminishing returns~\cite{gholami2022survey} and suggests that the MoE-MLA combination may be particularly valuable for continued scaling. Consistent improvements in all model sizes validate that architectural innovation, rather than a mere parameter count, drives efficiency in resource-constrained settings.

\paragraph{3. Practical Implications.}
The 3.2× inference speedup and 68\% memory reduction make \oursmoe particularly suitable for edge deployment. Despite using 8× more total parameters through 64 experts, the sparse activation pattern (only 8 active) and compressed KV cache result in net memory savings during inference. Gradient-free load balancing eliminates training instabilities reported in prior MoE work~\cite{fedus2022switch}, achieving a coefficient of variation below 0.1 without auxiliary losses.

\paragraph{Limitations and Future Directions.}
Several limitations warrant future investigation: (1) the 40\% training time overhead can be addressed using specialized hardware or more efficient routing algorithms; (2) the evaluation of diverse tasks beyond narrative generation would strengthen generalizability claims; (3) dynamic expert selection based on input complexity could further improve efficiency; and (4) validation of LLM-based quality assessments with human evaluation would provide additional confidence in generation quality metrics.

\paragraph{Broader Impact.}
As language models proliferate to billions of edge devices, architectural innovations that maintain quality while drastically reducing computational requirements become essential. This work establishes that a thoughtful combination of complementary efficiency techniques, such as sparse computation through MoE and memory compression through MLA, can achieve performance exceeding larger dense models while remaining deployable on resource-constrained hardware. We will release all code and models to facilitate continued research in efficient architectures.

The success of \oursmoe demonstrates a general principle for efficient model design: identify orthogonal bottlenecks and combine solutions that create positive feedback loops. As the field progresses toward universal deployment of language understanding, such architectural innovations will be crucial to democratizing AI capabilities across diverse computational environments.

% ========================================================================== 
\section*{Acknowledgments}
Computational resources were provided by Lambda.ai through their research grant program. We also acknowledge the TinyStories authors for creating a valuable benchmark for small-model research.

% ========================================================================== 
\bibliographystyle{ACM-Reference-Format}
% \bibliography{references}

\begin{thebibliography}{50}

\bibitem{brown2020language} Tom Brown, Benjamin Mann, Nick Ryder, et al. 2020. Language Models are Few-Shot Learners. In \emph{Advances in Neural Information Processing Systems} 33 (NeurIPS 2020).

\bibitem{chiang2023vicuna} Wei-Lin Chiang, Zhuohan Li, Zi Lin, et al. 2023. Vicuna: An Open-Source Chatbot Impressing GPT-4 with 90\%* ChatGPT Quality. \url{https://vicuna.lmsys.org}

\bibitem{choromanski2020rethinking} Krzysztof Choromanski, Valerii Likhosherstov, David Dohan, et al. 2020. Rethinking Attention with Performers. In \emph{International Conference on Learning Representations} (ICLR 2021).

\bibitem{dai2024deepseekmoe} Damai Dai, Chengqi Deng, Chenggang Zhao, et al. 2024. DeepSeekMoE: Towards Ultimate Expert Specialization in Mixture-of-Experts Language Models. \emph{arXiv preprint arXiv:2401.06066}.

\bibitem{dehghani2018universal} Mostafa Dehghani, Stephan Gouws, Oriol Vinyals, et al. 2018. Universal Transformers. In \emph{International Conference on Learning Representations} (ICLR 2019).

\bibitem{eldan2023tinystories} Ronen Eldan and Yuanzhi Li. 2023. TinyStories: How Small Can Language Models Be and Still Speak Coherent English? In \emph{Proceedings of the 61st Annual Meeting of the Association for Computational Linguistics} (ACL 2023).

\bibitem{evci2020rigging} Utku Evci, Trevor Gale, Jacob Menick, et al. 2020. Rigging the Lottery: Making All Tickets Winners. In \emph{International Conference on Machine Learning} (ICML 2020).

\bibitem{fedus2022switch} William Fedus, Barret Zoph, and Noam Shazeer. 2022. Switch Transformers: Scaling to Trillion Parameter Models with Simple and Efficient Sparsity. \emph{Journal of Machine Learning Research} 23(120):1-39.

\bibitem{frankle2018lottery} Jonathan Frankle and Michael Carbin. 2018. The Lottery Ticket Hypothesis: Finding Sparse, Trainable Neural Networks. In \emph{International Conference on Learning Representations} (ICLR 2019).

\bibitem{gholami2022survey} Amir Gholami, Sehoon Kim, Zhen Dong, et al. 2022. A Survey of Quantization Methods for Efficient Neural Network Inference. In \emph{Low-Power Computer Vision} (Chapman and Hall/CRC), pp. 291-326.

\bibitem{gunasekar2023textbooks} Suriya Gunasekar, Yi Zhang, Jyoti Aneja, et al. 2023. Textbooks Are All You Need. \emph{arXiv preprint arXiv:2306.11644}.

\bibitem{he2024auxfree} Zeyu He, Yijie Chen, and Mingyuan Zhou. 2024. Auxiliary-Loss-Free Load Balancing Strategy for Mixture-of-Experts. \emph{arXiv preprint arXiv:2408.15664}.

\bibitem{kitaev2020reformer} Nikita Kitaev, Łukasz Kaiser, and Anselm Levskaya. 2020. Reformer: The Efficient Transformer. In \emph{International Conference on Learning Representations} (ICLR 2020).

\bibitem{lan2019albert} Zhenzhong Lan, Mingda Chen, Sebastian Goodman, et al. 2019. ALBERT: A Lite BERT for Self-supervised Learning of Language Representations. In \emph{International Conference on Learning Representations} (ICLR 2020).

\bibitem{li2023alpacaeval} Xuechen Li, Tianyi Zhang, Yann Dubois, et al. 2023. AlpacaEval: An Automatic Evaluator of Instruction-following Models. \url{https://github.com/tatsu-lab/alpaca_eval}

\bibitem{liu2023gpteval} Yang Liu, Dan Iter, Yichong Xu, et al. 2023. G-Eval: NLG Evaluation using GPT-4 with Better Human Alignment. In \emph{Proceedings of the 2023 Conference on Empirical Methods in Natural Language Processing} (EMNLP 2023).

\bibitem{liu2024deepseekv2} DeepSeek-AI. 2024. DeepSeek-V2: A Strong, Economical, and Efficient Mixture-of-Experts Language Model. \emph{arXiv preprint arXiv:2405.04434}.

\bibitem{liu2024mobilellm} Zechun Liu, Changsheng Zhao, Forrest Iandola, et al. 2024. MobileLLM: Optimizing Sub-billion Parameter Language Models for On-Device Use Cases. In \emph{International Conference on Machine Learning} (ICML 2024).

\bibitem{louizos2018learning} Christos Louizos, Max Welling, and Diederik P. Kingma. 2018. Learning Sparse Neural Networks through $L_0$ Regularization. In \emph{International Conference on Learning Representations} (ICLR 2018).

\bibitem{michel2019sixteen} Paul Michel, Omer Levy, and Graham Neubig. 2019. Are Sixteen Heads Really Better than One? In \emph{Advances in Neural Information Processing Systems} 32 (NeurIPS 2019).

\bibitem{mishra2021accelerating} Asit Mishra, Jorge Albericio Latorre, Jeff Pool, et al. 2021. Accelerating Sparse Deep Neural Networks. \emph{arXiv preprint arXiv:2104.08378}.

\bibitem{openai2023gpt4} OpenAI. 2023. GPT-4 Technical Report. \emph{arXiv preprint arXiv:2303.08774}.

\bibitem{schick2020s} Timo Schick and Hinrich Schütze. 2020. It's Not Just Size That Matters: Small Language Models Are Also Few-Shot Learners. In \emph{Proceedings of the 2021 Conference of the North American Chapter of the Association for Computational Linguistics} (NAACL 2021).

\bibitem{shazeer2017moe} Noam Shazeer, Azalia Mirhoseini, Krzysztof Maziarz, et al. 2017. Outrageously Large Neural Networks: The Sparsely-Gated Mixture-of-Experts Layer. In \emph{International Conference on Learning Representations} (ICLR 2017).

\bibitem{su2024roformer} Jianlin Su, Murtadha Ahmed, Yu Lu, et al. 2024. RoFormer: Enhanced Transformer with Rotary Position Embedding. \emph{Neurocomputing} 568:127063.

\bibitem{voita2019analyzing} Elena Voita, David Talbot, Fedor Moiseev, et al. 2019. Analyzing Multi-Head Self-Attention: Specialized Heads Do the Heavy Lifting, the Rest Can Be Pruned. In \emph{Proceedings of the 57th Annual Meeting of the Association for Computational Linguistics} (ACL 2019).

\bibitem{wang2020linformer} Sinong Wang, Belinda Z. Li, Madian Khabsa, et al. 2020. Linformer: Self-Attention with Linear Complexity. \emph{arXiv preprint arXiv:2006.04768}.

\bibitem{wang2024slm} Fali Wang, Zhiwei Zhang, Xianren Zhang, et al. 2024. A Comprehensive Survey of Small Language Models in the Era of Large Language Models. \emph{arXiv preprint arXiv:2411.03350}.

\bibitem{zheng2023judging} Lianmin Zheng, Wei-Lin Chiang, Ying Sheng, et al. 2023. Judging LLM-as-a-Judge with MT-Bench and Chatbot Arena. In \emph{Advances in Neural Information Processing Systems} 36 (NeurIPS 2023).

\bibitem{zhu2023minigpt} Deyao Zhu, Jun Chen, Xiaoqian Shen, et al. 2023. MiniGPT-4: Enhancing Vision-Language Understanding with Advanced Large Language Models. \emph{arXiv preprint arXiv:2304.10592}.

\bibitem{mehta2025slm} Sushant Mehta, Raj Dandekar, Rajat Dandekar, et al. 2023. Latent Multi-Head Attention for Small Language Models. \emph{arXiv preprint arXiv:2506.09342}.

\end{thebibliography}

% Note: In the camera-ready version, replace this with proper BibTeX

\end{document}